\documentclass{article}

 \usepackage[preprint]{neurips_2025}

\usepackage[utf8]{inputenc} 
\usepackage[T1]{fontenc}    
\usepackage{hyperref}       
\usepackage{url}            
\usepackage{booktabs}       
\usepackage{amsfonts}       
\usepackage{nicefrac}       
\usepackage{microtype}      
\usepackage{xcolor}         
\usepackage{bm}
\usepackage{amsmath}
\usepackage{amssymb}
\usepackage{mathtools}
\usepackage{amsthm}
\usepackage{caption}
\theoremstyle{definition} 
\newtheorem{definition}{Definition}[section] 
\newtheorem{proposition}{Proposition}[section]
\newtheorem{theorem}{Theorem}[section]
\newtheorem{corollary}{Corollary}[section]
\newtheorem{lemma}{Lemma}[section]
\usepackage{mathrsfs}
\usepackage{diagbox}
\title{Beyond Surrogates: A Quantitative Analysis for Inter-Metric Relationships}

\author{%
  Yuanhao Pu$^{1,3}$, Defu Lian\thanks{Corresponding Author} $\ ^{2,3}$, Enhong Chen$^{2,3}$\\
  $^1$School of Artificial Intelligence \& Data Science, University of Science and Technology of China \\
  $^2$School of Computer Science \& Technology, University of Science and Technology of China\\
  $^3$ State Key Laboratory of Cognitive Intelligence, China\\
  \texttt{puyuanhao@mail.ustc.edu.cn, \{liandefu,cheneh\}@ustc.edu.cn} \\
}

\begin{document}

\maketitle

\begin{abstract}
The \textbf{Consistency} property between surrogate losses and evaluation metrics has been extensively studied to ensure that minimizing a loss leads to metric optimality. However, the direct relationship between different evaluation metrics remains significantly underexplored. This theoretical gap results in the "Metric Mismatch" frequently observed in industrial applications, where gains in offline validation metrics fail to translate into online performance. To bridge this disconnection, this paper proposes a unified theoretical framework designed to quantify the relationships between metrics. We categorize metrics into different classes to facilitate a comparative analysis across different mathematical forms and interrogates these relationships through \textbf{Bayes-Optimal Set} and \textbf{Regret Transfer}. Through this framework, we provide a new perspective on identifying the structural asymmetry in regret transfer, enabling the design of evaluation systems that are theoretically guaranteed to align offline improvements with online objectives.
\end{abstract}

\section{Introduction}
\label{sec:intro}

Modern machine learning heavily relies on the paradigm of metric-driven optimization. In a typical supervised learning scenario, researchers select an evaluation metric $\mathcal{M}$ that aligns with the practical objective, and minimize a differentiable surrogate loss $\mathcal{L}$ that serves as a proxy for $\mathcal{M}$. 
A fundamental assumption underlying this paradigm is \textit{monotonicity}: a reduction in the risk of $\mathcal{L}$ is expected to translate into an improvement in $\mathcal{M}$. This is theoretically supported by \textbf{Bayes-consistency} \citep{bartlett2006convexity,tewari2007consistency}, which guarantees that minimizing $\mathcal{L}$ leads to the Bayes-optimal decision for $\mathcal{M}$. Even under distribution shift, theoretical guarantees on domain-invariant learning\citep{blitzer2007domainadapt, bendavid2010domain} implicitly rely on the assumption that this structural consistency persists across domains, which is presumed to enable offline optimization to translate into performance gains in actual deployments.

A common expectation is that if multiple metrics $\{\mathcal{M}_A, \mathcal{M}_B,\cdots\}$ are all Bayes-consistent with the same surrogate $\mathcal{L}$, their empirical performance should be mutually aligned along the optimization trajectory. However, this idealization frequently collapses in industrial practice, leading to a \textbf{"Metric Mismatch"} for online application systems\citep{gilotte2018offline, krauth2020offline} . This misalignment stems from two fundamental hurdles: First, establishing Bayes-consistency is \textit{non-trivial}\citep{ravikumar2011ndcg, buffoni2011learning, calauzenes2012non} and often mathematically intractable for complex industrial losses (which incorporate various heuristics and regularizations) relative to metrics with intricate physical meanings. Second, Bayes-consistency is an \textit{asymptotic} property that fails to characterize the rate of convergence or the structural sensitivity of metrics\citep{ciliberto2020general}. Due to profound structural differences, metrics respond disparately to the minimization of $\mathcal{L}$; for instance, in ranking tasks\citep{rudin2009pnorm}, a model may achieve a higher AUC by refining the relative ordering of tail items, yet simultaneously degrade the top-ranked results, leading to a decline in top-heavy metrics like NDCG@$k$ or online click-through rates. 

While the Bayes-consistency between surrogate losses and target metrics has been extensively studied\citep{zhang2004statistical, zhang2004statistical2, agarwal2014surrogate,yang2020consistency, gao2015consistency}, theoretical exploration of the direct structural relationships between metrics remains underexplored. The primary obstacle lies in the fact that most evaluation metrics are non-smooth, discrete, and highly distribution-dependent, making it mathematically intractable to establish tight and transferable bounds directly between them. This inherent complexity has shifted the focus of the community toward surrogate losses, as they provide a more tractable alternative to direct metric optimization. However, in practical applications, metric mismatch is pervasive, and since indirect theoretical results centered on surrogate losses cannot be readily converted into tools for resolving direct metric conflicts, performance trade-offs between metrics still rely heavily on intuition or empirical A/B testing\citep{gomezuribe2016netflix}. This disconnection presents a critical dilemma for supervised learning applications: how to discern whether a model with superior offline gains (measured by a surrogate $\mathcal{L}$ or a baseline metric $\mathcal{M}_A$) will faithfully yield better online decisions ($\mathcal{M}_B$).

To bridge this gap, this paper proposes a unified theoretical framework based on metric structures to analyze the correlation and performance transfer behaviors among different metrics within the same task scenario. Specifically, we make the following contributions:
\begin{itemize} 
\item We draw inspiration from the classical framework of surrogate losses and categorize metrics into three distinct groups: \textbf{Pointwise, Pairwise,} and \textbf{Listwise}.
\item Within this taxonomy, we characterize the inclusion relationships between the sets of \textbf{Bayes-optimal} predictors for different metrics, establishing a structural foundation for inter-metric consistency. 
\item Recognizing that Bayes-optimality is rarely achievable in practice, we provide a quantitative analysis of \textbf{regret transfer} between metrics. We derive explicit bounds to determine how performance gains or losses in one metric translate to another.  
\end{itemize}
Consequently, we aim to answer a fundamental question for applications: \textit{If a model achieves an $\varepsilon$-regret on metric $\mathcal{M}_A$, what is the guaranteed upper bound of its regret on metric $\mathcal{M}_B$?} By addressing this question, we provide an analytical tool to navigate the performance trade-offs that have traditionally relied on human intuition or costly empirical testing. The theoretical landscape of inter-metric relationships is summarized in Table \ref{tab:metric_matrix}.

\begin{table}[h]
\centering
\caption{Cross-Analysis of Metric Relationships: From Consistency to Mismatch.}
\label{tab:metric_matrix}
\renewcommand{\arraystretch}{1.3} 
\begin{tabular}{c|c|c|c}
\toprule
\diagbox{\textbf{Source}}{\textbf{Target}} & Pointwise & Pairwise & Listwise \\ \midrule
Pointwise & \begin{tabular}[c]{@{}c@{}}\textbf{Cohesion}\\ Theorem~\ref{thm:intra_alignment} \& \ref{thm:truncation_monotonicity}\end{tabular} & \begin{tabular}[c]{@{}c@{}}\textbf{Transfer Failure}\\ Theorem~\ref{thm:inter_group_relations} \&\ref{thm:pointwise_failure}\end{tabular} & \begin{tabular}[c]{@{}c@{}}\textbf{Transfer Failure}\\ Theorem~\ref{thm:inter_group_relations} \&\ref{thm:pointwise_failure}\end{tabular} \\ \hline
Pairwise & \begin{tabular}[c]{@{}c@{}}\textbf{Stable Transfer}\\ Theorem~\ref{thm:inter_group_relations} \& \ref{thm:ranking_to_pointwise}\end{tabular} &  \begin{tabular}[c]{@{}c@{}}\textbf{Cohesion}\\ Theorem~\ref{thm:intra_alignment} \& \ref{thm:truncation_monotonicity} \end{tabular}& \begin{tabular}[c]{@{}c@{}}\textbf{Weak Transfer}\\ Theorem~\ref{thm:inter_group_relations} \& \ref{thm:rates}\end{tabular} \\ \hline
Listwise & \begin{tabular}[c]{@{}c@{}}\textbf{Strong Transfer}\\ Theorem~\ref{thm:inter_group_relations} \& \ref{thm:ranking_to_pointwise}\end{tabular} & \begin{tabular}[c]{@{}c@{}}\textbf{Strong Transfer}\\ Theorem~\ref{thm:inter_group_relations} \& \ref{thm:rates}\end{tabular} & \begin{tabular}[c]{@{}c@{}}\textbf{Cohesion}\\ Theorem~\ref{thm:intra_alignment} \& \ref{thm:truncation_monotonicity} \end{tabular} \\ 
\bottomrule
\end{tabular}
\end{table}
\section{Related Works}

\subsection{Consistency: From Metrics to Surrogate Optimization}
The study of statistical (Bayes-)consistency centers on whether minimizing a surrogate loss ensures the minimization of the target metric risk. Foundational works by \cite{zhang2004statistical, zhang2004statistical2},\cite{bartlett2006convexity} and \cite{steinwart2007compare} provided the first characterization of consistency for convex surrogates in binary classification, which was later extended to multiclass by \cite{tewari2007consistency}, revealing the increased complexity of calibration conditions beyond binary labels. A deeper theoretical framework \citet{reid2011information} connected surrogate losses with consistency properties and divergence measures, offering a comprehensive perspective on risk bounds. Soon afterwards, researchers shifted focus toward task-specific consistency. including regression\citep{christmann2007consistency}, top-$k$ classification \citep{yang2020consistency} and ranking problems \citep{cossock2006subset,ravikumar2011ndcg,gao2015consistency}. 

More recently, a significant series of works\citep{long2013consistency,zhang2020bayes,awasthi2022hconsistency} introduced $\mathcal{H}$-consistency to address the limitations of traditional Bayes-consistency. Unlike classical theory which assumes an arbitary hypothesis space, $\mathcal{H}$-consistency provides non-asymptotic guarantees on specific hypothesis set $\mathcal{H}$ (e.g., neural networks), which shows awareness of different model structures. Consequently, the consistency framework establishes a convergence guarantee from surrogate loss function $\mathcal{L}$ to task objective metric $\mathcal{M}$, providing a robust theoretical explanation for ML training paradigm.

\subsection{Analysis of Evaluation Metrics}

Though evaluation metrics are often treated as static measurement tools, their mathematical properties and formal foundations constitute an independent field of theoretical study. To formalize metric selection, an axiomatic approach\citep{amigo2009comparison} introduced a framework based on formal constraints to assess metric robustness. This perspective was recently developed by \cite{gosgens2021good}, which analyzed a series of ideal properties including monotonicity and symmetry to filter classification measures, shifting metric selection from heuristic choices toward logical principles. In addition, the statistical behavior of metrics has been discussed through the lens of consistency. For instance, \cite{wang2013theoretical} provided a detailed theoretical analysis of NDCG-type measures under various decay factors, while \cite{narasimhan2014statistical} established the consistency of plug-in classifiers for non-decomposable measures. For more abstract unifications, \citet{hernandez2012unified} translated diverse metrics into a common expression of expected classification loss. More recently, \cite{balestra2024ranking} introduced a group-theoretic perspective to ranking metrics, revealing the underlying algebraic structures and symmetries within the permutation space.

Despite these advancements, the analysis of metrics remains largely decoupled and lacks a unified framework. Most existing works either analyze the properties of a single metric or investigate its consistency within a specific task. Building upon this, our work offers a significant advantage by providing a unified theoretical framework that bridges the gap between various metrics. 
\section{Preliminaries}

Let $\mathscr{X} \subseteq \mathbb{R}^d$ denote the input feature space, and let $\mathscr{Y}$ denote the output label space. Depending on the specific classification/ranking scenario, $\mathscr{Y}$ may consist of binary labels $\{0, 1\}$ or multi-level relevance scores $\{0, 1, \dots, y_{\mathrm{max}}\}$. However, real-world applications (e.g., online recommender systems) typically rely on sparse interactions that are inherently binary or purposefully binarized for efficiency. 
We assume a joint distribution $\mathbb{P}$ over $\mathscr{X} \times \mathscr{Y}$. A predictor is a scoring function $f: \mathscr{X} \to \mathbb{R}$ belonging to a hypothesis class $\mathscr{F}$. For any instance $\bm{x} \in \mathscr{X}$, $f(\bm{x})$ produces a real-valued score determining predicted label or relative rank of the item within a specific output list.
\subsection{Surrogate Risk and Metric Risk}
\label{subsec:metric}
In typical supervised learning, we minimize a differentiable surrogate loss $\mathcal{L}: \mathbb{R} \times \mathscr{Y} \to \mathbb{R}_+$. The expected surrogate risk of a predictor $f$ is defined as:
\begin{equation}
    R_{\mathcal{L}}(f) = \mathbb{E}_{(\bm{x}, y) \sim \mathbb{P}} [\mathcal{L}(f(\bm{x}), y)]
\end{equation}
The ultimate objective, however, is to optimize an evaluation metric $\mathcal{M}$. Unlike surrogate losses, classification or ranking metrics are defined over a sequence of instances $\mathscr{D} = \{(\bm{x}_i, y_i)\}_{i=1}^n$. Let $\sigma_f$ be a permutation of $\{1, \dots, n\}$ induced by sorting the scores $f(\bm{x}_i)$ in descending order, s.t. 
\begin{equation}
    f(\bm{x}_{\sigma_f(1)}) \ge f(\bm{x}_{\sigma_f(2)}) \ge \dots \ge f(\bm{x}_{\sigma_f(n)}).
\end{equation}
We define the general form of the metric risk $R_{\mathcal{M}}(f)$ as:
\begin{equation}
    R_{\mathcal{M}}(f) = 1 - \mathbb{E}_{\mathscr{D} \sim \mathbb{P}^n} \left[ \sum_{i=1}^{n} w(i, y_{\sigma_f(i)}, \mathscr{D}) \cdot \phi(y_{\sigma_f(i)}) \right]
\end{equation}
where $w(\cdot)$ is a position-dependent weight function and $\phi(\cdot)$ is a utility function (typically $\phi(y)=y$ for binary relevance). This unified framework encapsulates a wide range of metrics used in both academic research and industrial production, whose detailed mathematical definitions are provided in Appendix~\ref{app:metric_definitions}. Specifically, we draw inspiration from the classical categories of surrogate losses\citep{liu2011learningtorank} and categorize these evaluation metrics into three distinct structural groups based on their evaluation behaviors:
\begin{itemize}
    \item \textbf{Pointwise} metrics ($\mathcal{G}_P$) focus on independent item identification, treating each instance as a classification or regression task. Examples include Accuracy (\textbf{Acc}) and truncated measures such as \textbf{Precision@$k$} and \textbf{Recall@$k$}.
    \item \textbf{Pairwise} metrics ($\mathcal{G}_R$) measure the relative ordering of item pairs, primarily focusing on global ranking quality. Representative example is the Area Under the ROC Curve (\textbf{AUC}), which assesses the probability of correct pairwise ranking.
    \item \textbf{Listwise} metrics ($\mathcal{G}_L$) are position-sensitive measures that evaluate the entire ranked list, assigning higher utility to items at top. This group includes \textbf{(N)DCG}, \textbf{MAP}, \textbf{MRR}, and their Top-$k$ truncated versions (\textbf{NDCG@$k$}, \textbf{MAP@$k$}, \textbf{MRR@$k$}).
\end{itemize}
These metrics encompass the most common classification and ranking criteria used in practice; notably, under the binary hypothesis $y \in \{0,1\}$, they are well-defined and free from the ambiguity encountered in multi-level settings.

\subsection{Bayes-Optimality and Regret}

To characterize the fundamental limits of performance, we define the optimal behavior of a predictor through functional sets.

\begin{definition}[Bayes-Optimal Predictor Set]
For a given evaluation metric $\mathcal{M}$ and a hypothesis class $\mathscr{F}$, the \textbf{Bayes-optimal predictor set} $\mathscr{F}^*_{\mathcal{M}}$ is defined as the collection of all scoring functions that achieve the infimum of the metric risk $R_{\mathcal{M}}$:
\begin{equation}
    \mathscr{F}^*_{\mathcal{M}} \coloneqq \left\{ f \in \mathscr{F} \mid R_{\mathcal{M}}(f) = \inf_{f' \in \mathscr{F}} R_{\mathcal{M}}(f') \right\}
\end{equation}
\end{definition}
Unlike regression tasks where the optimal solution is often unique, classification or ranking metrics may induce a high-dimensional equivalence class of optimal scoring functions. This set-based formulation allows us to categorize the structural compatibility between any $\mathcal{M}_A$ and $\mathcal{M}_B$.

\begin{definition}[Bayes-Optimal Inclusion and Equivalence]
We denote the structural relationship between metrics as follows:
\begin{itemize}
    \item \textbf{Bayes-Subsumed ($\preceq_{\mathscr{B}}$):} We write $\mathcal{M}_A \preceq_{\mathscr{B}} \mathcal{M}_B$ if $\mathscr{F}^*_{\mathcal{M}_A} \subseteq \mathscr{F}^*_{\mathcal{M}_B}$, which implies that any predictor achieving theoretical optimality for $\mathcal{M}_A$ is guaranteed to be optimal for $\mathcal{M}_B$.
    \item \textbf{Bayes-Equivalent ($\equiv_{\mathscr{B}}$):} We write $\mathcal{M}_A \equiv_{\mathscr{B}} \mathcal{M}_B$ if $\mathscr{F}^*_{\mathcal{M}_A} = \mathscr{F}^*_{\mathcal{M}_B}$, indicating they share identical theoretical goals.
\end{itemize}
\end{definition}
\begin{definition}[Metric Regret]
For a given predictor $f \in \mathscr{F}$, the \textbf{Regret} with respect to metric $\mathcal{M}$ is defined as the excess risk over the Bayes-optimal value:
\begin{equation}
    \text{Regret}_{\mathcal{M}}(f) \coloneqq R_{\mathcal{M}}(f) - \min_{f^* \in \mathscr{F}^*_{\mathcal{M}}} R_{\mathcal{M}}(f^*)
\end{equation}
By construction, $\text{Regret}_{\mathcal{M}}(f) \ge 0$, and $\text{Regret}_{\mathcal{M}}(f) = 0$ iff $f \in \mathscr{F}^*_{\mathcal{M}}$.
\end{definition}

This set-based perspective reveals the relationship between surrogate optimization and metric.

\begin{corollary}[Bayes-Consistency]
A surrogate loss $\mathcal{L}$ is \textbf{Bayes-consistent} with a metric $\mathcal{M}$ if, for any sequence of predictors $\{f_k\}_{k=1}^{\infty}$, the vanishing of the surrogate regret implies the vanishing of the metric regret:
\begin{equation}
    \text{Regret}_{\mathcal{L}}(f_k) \xrightarrow{k \to \infty} 0 \implies \text{Regret}_{\mathcal{M}}(f_k) \xrightarrow{k \to \infty} 0
\end{equation}
\end{corollary}
This formulation provides the necessary granularity to analyze why a surrogate loss might be simultaneously consistent with multiple metrics. Specifically, for any two metrics satisfying $\mathcal{M}_A \equiv_{\mathscr{B}} \mathcal{M}_B$, a surrogate loss consistent with $\mathcal{M}_A$ is inherently consistent with $\mathcal{M}_B$. However, this asymptotic equivalence does not guarantee a uniform \textbf{regret transfer}. In practical optimization, the rate of convergence and the sensitivity to local ranking perturbations can differ fundamentally between metrics, as governed by the functional form of their respective weight functions $w(i)$.

\subsection{Regret Transfer}

In practice, achieving the Bayes-optimal state $f^* \in \mathscr{F}^*_{\mathcal{M}}$ is often hindered by finite sample size, restricted model capacity, and the complexities of the optimization landscape. Consequently, it is more pragmatic to analyze the behavior of metrics when the predictor $f$ resides within an $\varepsilon$-neighborhood of the optimal frontier. To quantify how the approximation error in one metric propagates to another, we introduce the following definition:

\begin{definition}[Regret Transfer Function]
For any two metrics $\mathcal{M}_A$ and $\mathcal{M}_B$, the \textbf{Regret Transfer Function} $\Psi_{A \to B}: [0, 1] \to [0, 1]$ characterizes the worst-case regret on metric $\mathcal{M}_B$ given an upper bound $\varepsilon$ on the regret of $\mathcal{M}_A$:
\begin{equation}
    \Psi_{A \to B}(\varepsilon) \coloneqq \sup_{f \in \mathscr{F}} \{ \text{Regret}_{\mathcal{M}_B}(f) \mid \text{Regret}_{\mathcal{M}_A}(f) \le \varepsilon \}
\end{equation}
\end{definition}
This formulation provides a unified perspective to analyze the sensitivity and robustness of metric performance in non-asymptotic regimes. It is worth noting that our definition departs from the classical calibration functions (or $\psi$-transforms) studied in \cite{agarwal2014surrogate,kotlowski2017surrogate}, which focus on the lower-bound relationship within $\mathcal{L}\to\mathcal{M}$. By characterizing the worst-case degradation of $\mathcal{M}_B$ relative to the approximation error $\varepsilon$ in $\mathcal{M}_A$, the transfer function $\Psi_{A \to B}(\varepsilon)$ serves as a theoretical bound on the collateral damage incurred by target metric. Such a mapping is particularly crucial in practical training process where Bayes-optimality is unattainable, as it evaluates the reliability of a proxy $\mathcal{M}_A$ under fixed computational or data constraints.

\section{Main Results}
\label{sec:main_results}

\subsection{Intra-group Cohesion}
\label{subsec:taxonomy_cohesion}

By categorizing evaluation metrics into Pointwise ($\mathcal{G}_{P}$), Pairwise ($\mathcal{G}_{R}$), and Listwise ($\mathcal{G}_{L}$) groups, our framework reveals an \textbf{intra-group convergence} property. This property suggests that metrics sharing the same evaluation unit exhibit highly aligned theoretical behaviors.

\begin{theorem}[Alignment and One-way Inclusion]
\label{thm:intra_alignment}
For any two metrics $\mathcal{M}_A, \mathcal{M}_B$ within the same structural group $\mathcal{G}_i$, their relationships are governed by their evaluation scope:
\begin{itemize}
    \item \textbf{Equivalence:} If $\mathcal{M}_A$ and $\mathcal{M}_B$ are both global (e.g., $\text{NDCG}, \text{MAP} \in \mathcal{G}_L$) or both truncated at the same $k$ (e.g., $\text{Pre@}k, \text{Rec@}k \in \mathcal{G}_P$), they are Bayes-equivalent, i.e., $\mathcal{M}_A \equiv_{\mathscr{B}}\mathcal{M}_B$ and well-defined regret transfer, i.e., $\Psi_{A \to B}(\varepsilon) \le C \cdot \varepsilon$.
    \item \textbf{Inclusion:} If $\mathcal{M}_A$ is global and $\mathcal{M}_B$ is truncated at $k$, there exists a one-way inclusion, s.t. $\mathcal{M}_A  \preceq_{\mathscr{B}} \mathcal{M}_B$. This implies that achieving global optimality guarantees local optimality, while the regret transfer $\Psi_{A \to B}(\varepsilon)$ remains well-defined.
\end{itemize}
\end{theorem}

The second dimension of intra-group cohesion concerns the stability of a single metric across different truncation depths. As $k$ varies, the stringency of the optimality requirement changes, leading to a nested structure of the Bayes-optimal sets.

\begin{theorem}[Truncation Monotonicity]
\label{thm:truncation_monotonicity}
For a truncated metric $\mathcal{M}@k \in \mathcal{G}_i$, its theoretical properties satisfy the following monotonicity with respect to $k$:
\begin{itemize}
    \item \textbf{Optimal Set:} For $k_1 < k_2$, the Bayes-optimal sets are nested, s.t. $\mathcal{M}{@k_2}  \preceq_{\mathscr{B}}\mathcal{M}{@k_1}$. A smaller $k$ induces a larger equivalence class of optimal predictors, as the metric is indifferent to ranking errors beyond the top-$k$ positions.
    \item \textbf{Regret Transfer:} The regret transfer between different truncation levels $\Psi_{k_2 \to k_1}(\varepsilon)$ is well-behaved, whereas the reverse transfer $\Psi_{k_1 \to k_2}(\varepsilon)$ exhibits increasing slack as the gap between $k_1$ and $k_2$ grows, reflecting the loss of information due to tighter truncation.
\end{itemize}
\end{theorem}

This hierarchical cohesion allows us to simplify the subsequent analysis: since intra-group behaviors are largely consistent or predictably nested, the metric mismatch observed in practice must stem from the fundamental structural divergence between these major groups in the non-asymptotic regime. Proofs are available in Appendix~\ref{appdix:proof_intra}.

\subsection{Inter-group Bayes-Optimal Convergence and Hierarchy}
\label{subsec:inter_group_bayes}

Building upon the intra-group analysis, we now examine the relationships between the Bayes-optimal sets across structural groups. For simplicity, we focus on global metrics Acc, AUC and NDCG as representatives, while broader results on other global and truncated metrics can be naturally derived by composing these results with the previous section. Detailed proof are in Appendix~\ref{appdix:inter_bayes}.

\begin{theorem}[Bayes-Optimal Set Relations across Groups]
\label{thm:inter_group_relations}
Let $\mathscr{F}^*_P$, $\mathscr{F}^*_R$, and $\mathscr{F}^*_L$ denote the Bayes-optimal sets for Pointwise, Pairwise, and Listwise groups, respectively. Under a general distribution $\mathbb{P}$ with binary relevance:
\begin{enumerate}
    \item \textbf{Pointwise vs. Listwise/Pairwise ($\mathcal{M}_L/\mathcal{M}_R \preceq_{\mathscr{B}} \mathcal{M}_P$):} The Bayes-optimal set for Acc is a superset of pairwise/listwise optimal sets. Any $f \in \mathscr{F}^*_L$ is necessarily in $\mathscr{F}^*_P$, but a predictor that minimizes classification risk may incur maximum ranking regret as $\mathcal{G}_P$ is indifferent to the ordering of instances on the same side of the threshold.
    \item \textbf{Pairwise vs. Listwise Convergence ($\mathcal{M}_L \equiv_{\mathscr{B}} \mathcal{M}_R$):} Both Pairwise (AUC) and Listwise metrics are maximized if and only if the predictor $f$ preserves the partial ordering of the conditional expectation $\eta(\bm{x})$. 
\end{enumerate}
\end{theorem}

\subsection{Inter-group Regret Transfer Results}
\label{subsec:inter_group_transfer}

Despite sharing the same Bayes-optimal frontier $\mathscr{F}^*$, the objective functions of pairwise and listwise metrics exhibit fundamental heterogeneity in their regret landscapes. AUC assigns a uniform weight $1/(n^+n^-)$ to all pairs, while listwise metrics concentrate their utility on the top positions of the permutation. Building upon the Bayes-optimal hierarchy established in Section \ref{subsec:inter_group_bayes}, we investigate the stability of regret transfer between metric groups. 


It is important to note that our objective is not merely to construct pathological counter-examples to negate this property (i.e., by showing $\text{Regret}_A \to 0$ while $\text{Regret}_B \ge Const > 0$). For instance, a common extreme case involves one positive item and $n^- \to \infty$ negative items. If the positive item is predicted at the second position:
\begin{equation}
    \text{Regret}_{\mathrm{AUC}} = \frac{1}{n^-} \to 0, \quad \text{Regret}_{\mathrm{NDCG}} = \frac{1}{\log 2} - \frac{1}{\log 3} > 0.
\end{equation}
While such a construction formally negates the transfer property, these extreme scenarios are of limited practical reference for model training where the counts of positive and negative samples ($n^+, n^-$) are typically fixed during training and evaluation.

\subsubsection{Pointwise $\leftrightarrow$ Pairwise/Listwise}

The transfer behavior between Pointwise classification ($\mathcal{G}_P$) and ranking groups ($\mathcal{G}_R, \mathcal{G}_L$) is inherently asymmetric, reflecting the hierarchical inclusion of their Bayes-optimal sets. The transfer from $\mathcal{G}_P$ to any ranking group is fundamentally divergent due to the lack of intra-class constraints.
\begin{theorem}[Pointwise Transfer Failure]
\label{thm:pointwise_failure}
$\Psi_{P \to R}(\epsilon)$ and $\Psi_{P \to L}(\epsilon)$ satisfy $\Psi(0) > 0$. 
\end{theorem}

\begin{proof}
Consider a sequence $\mathscr{D}$ containing at least two instances within the same class (e.g., $\eta_i, \eta_j > 0.5$ with $\eta_i > \eta_j$). For any predictor $f$ that assigns scores $s_i, s_j > 0.5$, the classification regret contribution from these items is zero:
\begin{equation}
    |\mathbb{I}(\eta_i > 0.5) - \mathbb{I}(s_i > 0.5)| + |\mathbb{I}(\eta_j > 0.5) - \mathbb{I}(s_j > 0.5)| = 0
\end{equation}
However, $\text{Regret}_{\mathrm{Acc}}$ is invariant to the relative ordering of $s_i$ and $s_j$. If $f$ predicts $s_j > s_i$, it incurs a non-zero ranking regret proportional to $(\eta_i - \eta_j)$ in both Pairwise and Listwise functionals, thus $\Psi_{P \to R/L}(0) \ge Const > 0$. The Pointwise functional is blind to any utility derived from ordering instances that fall on the same side of the decision threshold.
\end{proof}

Conversely, the transfer from ranking groups to Pointwise classification is straightforward.
\begin{theorem}[Transfer from Ranking to Pointwise]
\label{thm:ranking_to_pointwise}
Let $\delta = \min_i |\eta_i - 0.5| > 0$ be the margin of the relevance distribution. For a sequence $\mathscr{D}$ with $n$ instances ($n^+$ positive, $n^-$ negative), the transfer functions satisfy:
\begin{align}
    \Psi_{\mathrm{AUC} \to \mathrm{Acc}}(\epsilon) &= \left( \frac{n^+ n^-}{n \delta} \right) \cdot \epsilon \\
    \Psi_{\mathrm{NDCG} \to \mathrm{Acc}}(\epsilon) &= \left( \frac{\mathrm{IDCG}_n}{n \delta \cdot w(n)} \right) \cdot \epsilon
\end{align}
\end{theorem}

This stability arises because any predictor that minimizes ranking regret must correctly separate instances with higher $\eta$ from those with lower $\eta$, whose detailed proof is in Appendix~\ref{appdix:ranking_to_pointwise}. 

\textit{Remark}: The dependency on $1/\delta$ reflects the inherent difficulty of classification near the decision boundary. In scenarios where $\delta\to 0$, the transfer bound can be further refined by adopting low-noise conditions\citep{clemenccon2008ranking, agarwal2014surrogate}, substituting the hard margin with a density-based control over the mass around $\eta(x)=0.5$. This ensures that the stability of ranking-to-pointwise transfer remains theoretically grounded even in non-separable domains.

\subsubsection{Asymmetry between Pairwise and Listwise Transfer}

A critical finding is the fundamental asymmetry on regret between $\mathcal{G}_R$ (AUC) and $\mathcal{G}_L$ (NDCG). 

\begin{theorem}[Asymmetry of Regret Transfer]
\label{thm:asymmetry}
Consider a ranking list of size $n$ with $n^+$ positive and $n^-$ negative instances. Let $w(r) = 1/\log(1+r)$ be the weighting function, with $\Delta_{\max} = w(1) - w(2)$ and $\Delta_{\min} = w(n-1) - w(n)$ denoting its maximum and minimum differential. The transfer functions $\Psi$ between Pairwise regret and Listwise regret satisfy:
    \begin{equation}
    \begin{aligned}
        \Psi_{\mathrm{AUC} \to \mathrm{NDCG}}(\epsilon) = \left( \frac{\Delta_{\max} \cdot n^+ n^-}{\sum_{i=1}^{n^+} w(i)} \right) \cdot \epsilon,\quad 
        \Psi_{\mathrm{NDCG} \to \mathrm{AUC}}(\epsilon) = \left( \frac{\sum_{i=1}^{n^+} w(i)}{n^+ n^- \cdot \Delta_{\min}} \right) \cdot \epsilon
    \end{aligned}
    \end{equation}
\end{theorem}

Detailed proof is in Appendix~\ref{appdix:proof_asymmetry}. It demonstrates that the coupling strength between metric groups is not a static constant, but a dynamic property governed by the list scale $n$ and the label density. In industrial applications, the transition under extremely sparse production logs often leads to a severe amplification of the transfer gap. To provide a theoretical basis for this circumstance, we now analyze the asymptotic growth rates of the transfer coefficients under representative label distributions.

\begin{theorem}[Asymptotic Transfer Rates]
\label{thm:rates}
The asymptotic growth rates of the transfer coefficients $C_{R \to L}$ and $C_{L \to R}$ exhibit distinct behaviors under the following label density scenarios:
\begin{enumerate}
    \item Balanced $n^+, n^- \sim O(n)$:
    \begin{equation}
        C_{R \to L} \sim O(n \log n), \quad C_{L \to R} \sim O(\log n).
    \end{equation}
    
    \item Imbalanced $n^+ \sim O(1), n^- \sim O(n)$:
    \begin{equation}
        C_{R \to L} \sim O(n), \quad C_{L \to R} \sim O(\log^2 n).
    \end{equation}
\end{enumerate}
\end{theorem}

Detailed proof is in Appendix~\ref{appdix:proof_rates}. Theorem \ref{thm:rates} reveals a fundamental scaling asymmetry between Pairwise and Listwise metrics as the system size $n$ increases. Specifically, $C_{R\to L}$ grows polynomially while $C_{L\to R}$ grows only logarithmically. This confirms that optimizing NDCG imposes a tighter constraint on AUC and provides a more robust guarantee for ranking quality. Conversely, relying on AUC as a proxy for top-heavy online objectives faces error amplification. Beyond Theorem~\ref{thm:rates}, the transfer coefficients $C_{R\to P}$ and $C_{L\to P}$ in Theorem~\ref{thm:ranking_to_pointwise} can also be asymptotically estimated, which reveals a similar pattern where NDCG impose a significantly tighter constraint on Acc than AUC:
\begin{enumerate}
    \item Balanced $n^+, n^- \sim O(n)$:
\begin{equation}
C_{R \to P} \sim O\left(n\right), \quad C_{L\to P} \sim O\left(1\right).
\end{equation}
    \item Imbalanced $n^+ \sim O(1), n^- \sim O(n)$:
    \begin{equation}
C_{R \to P} \sim O\left(1\right), \quad C_{L\to P} \sim O\left(\frac{\log n}{n}\right).
\end{equation}
\end{enumerate}

This explains why offline AUC gains often vanish in online production.

\section{Experiments}

In this section, we conduct controlled simulations and real-world experiments to empirically validate our theoretical framework. We focus on (\textit{i}) the Pointwise transfer failure and (\textit{ii}) the regret scaling asymmetry between Pairwise and Listwise metrics by adopting a two-tier experimental design:

\begin{itemize}
    \item \textbf{Structural Simulation}: Since the theoretical results represent the worst-case behavior of the regret manifold, we utilize structural simulation to probe regret transfer, which allows us to inject specific error patterns of different loss functions.
    \item \textbf{Real-world Experiments}: We utilize real-world optimization tasks of representative losses (BCE, BPR\cite{rendle2009bpr}, ListNet\cite{cao2007learningtorank}) and report the direct performance comparisons across different metrics to corroborate our theoretical findings.
\end{itemize}

\subsection{Regret Manifold Simulation}

\textbf{Data Generation.} To evaluate the scaling properties, we simulate a ranking environment with $n = 1,000$ items per list. We generate the ground-truth Bayes-optimal relevance scores $\eta \in [0.01, 0.99]^n$ sampled from a uniform distribution. This construction allows for exact calculation of the Bayes-optimal performance $M(f^*_{\text{Bayes}})$, serving as the denominators for our regret estimations.

\textbf{Loss-Specific Error Modeling.} Rather than relying solely on the stochasticity of neural training, we explicitly model the prediction function $f$ to characterize the intrinsic behavior of different losses:
\begin{itemize}
    \item \textbf{Pointwise}: We simulate a classifier that resolves the decision boundary at $0.5$ but remains agnostic to intra-class ordering. The error is modeled as random perturbations that preserve the class label but ignore relative ranks within the same class.
    \item \textbf{Pairwise}: We model a globally consistent ranker that is insensitive to position. We introduce an error probe by swapping elements at the top of the list, which results in a negligible impact on AUC regret but a significant impact on NDCG regret.
    \item \textbf{Listwise}: We model a position-aware ranker where the noise variance is suppressed by a factor of $1/\log(i+1)$ at rank $i$. This reflects the strong constraint listwise objectives exert on the head of the distribution.
\end{itemize}

\textbf{Regret Visualization and Metrics.} To capture the full spectrum of the regret manifold, we generate 500 snapshots for each loss type by varying a convergence parameter $\alpha \in [0, 1]$, where $f \to \eta$ as $\alpha \to 1$. For each snapshot, we estimate the regret for metric $M$ as:
$$ \text{Regret}_M(f) = 1 - \frac{M(f)}{M(f^*_{\text{Bayes}})} $$
We visualize these snapshots in a 3D Regret Space ($\text{Regret}_{\text{Acc}} \times \text{Regret}_{\text{AUC}} \times \text{Regret}_{\text{NDCG}}$) and cut a slice (NDCG vs. Acc) to observe the global geometry, as shown in Figure~\ref{fig:regret_clouds}. Besides, we examine the average regret between $\text{Regret}_{\text{AUC}}$ and $\text{Regret}_{\text{NDCG}}$ to empirically verify the scaling asymmetry between pairwise and listwise stability in Table~\ref{tab:regret_metrics}.

\begin{figure*}[t]
    \centering
    \begin{minipage}{1.0\linewidth}
        \centering
        \includegraphics[width=0.85\linewidth]{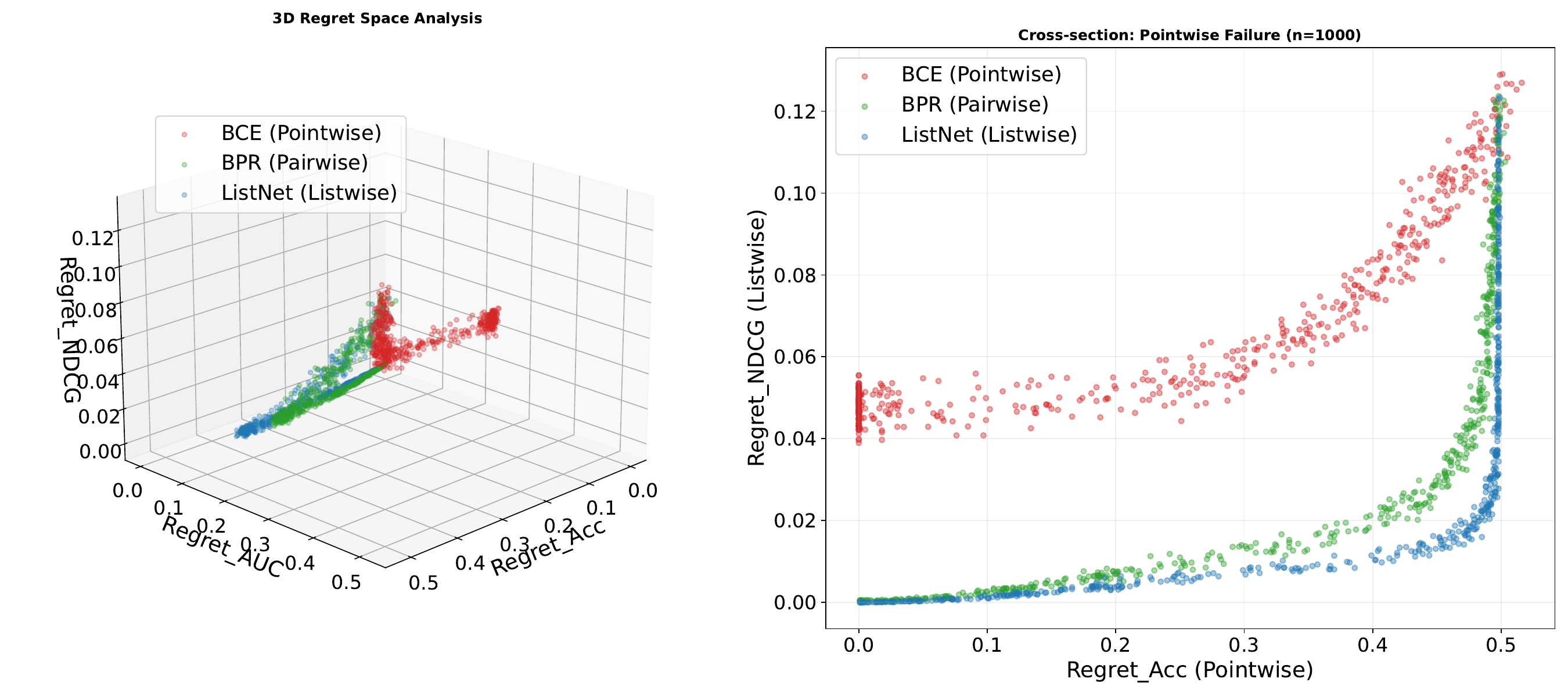} 
        \caption{Regret space visualization for Simulated Pointwise, Pairwise, and Listwise losses ($n=1,000$). Left: 3D regret landscape. Right: Cross-section of Acc vs. NDCG regret.}
        \label{fig:regret_clouds}
    \end{minipage}

    \vspace{15pt} 

    \footnotesize 
    \begin{minipage}[t]{0.48\linewidth}
        \centering
        \captionof{table}{Simulation Regret Comparison}
        \label{tab:regret_metrics}
        \vspace{5pt}
        \begin{tabular}{lccc}
            \toprule
            \textbf{Loss} & $r_{\text{acc}}$ & $r_{\text{auc}}$ & $r_{\text{ndcg}}$ \\
            \midrule
            Pointwise     & \textbf{0.2325} & 0.3270 & 0.0667 \\
            Pairwise     & 0.3329 & 0.2104 & 0.0361 \\
            Listwise & 0.3405 & \textbf{0.1756} & \textbf{0.0283} \\
            \bottomrule
        \end{tabular}
    \end{minipage}
    \hfill
    \begin{minipage}[t]{0.48\linewidth}
        \centering
        \captionof{table}{Real-world Performance on ML-1M}
        \label{tab:ml1m_results}
        \vspace{5pt}
        \begin{tabular}{lccc}
            \toprule
            \textbf{Loss} & Recall@10 & AUC & NDCG@10 \\
            \midrule
            BCE     & 0.1574 & 0.9047 & 0.2140 \\
            BPR     & 0.1871 & \textbf{0.9212} & 0.2349 \\
            ListNet & \textbf{0.1935} & 0.9207 & \textbf{0.2434} \\
            \bottomrule
        \end{tabular}
    \end{minipage}
\end{figure*}

\subsection{Real-World Data Experiments}

In the real-world setting, we evaluate our theoretical framework on the \textbf{MovieLens-1M} dataset. Since the ground-truth Bayes-optimal probability $\eta$ is latent in real datasets, we directly report the observed performance across standard metrics instead of calculating regrets. This allow us to verify if the optimization efficiency of different losses aligns with our theoretical stability bounds. The experimental results are summarized in Table~\ref{tab:ml1m_results}.

\subsection{Discussion}

The empirical results from both simulations and real-world experiments provide a consistent narrative that validates our theoretical framework. As shown in Figure~\ref{fig:regret_clouds} and Table~\ref{tab:regret_metrics}, we observe a distinct Pointwise Transfer Failure. While the Pointwise setting achieves the lowest classification regret ($r_{\text{acc}} = 0.2325$), it consistently suffers from the highest ranking regrets. This disconnect underscores the fact that minimizing pointwise classification error is insufficient for high-quality ranking. 

Regarding the asymmetric transfer, although the structural simulations capture the directional trend where the average $r_{\text{auc}}$ and $r_{\text{ndcg}}$ for Listwise setting is notably lower than that of the Pairwise objective, the fixed environment of the simulation may not fully demonstrate the scaling divergence between these two paradigms. Instead, the real-world performance provides a more pronounced demonstration of this structural gap. As shown in Table~\ref{tab:ml1m_results}, while BPR achieves a marginally superior AUC ($0.9212$), reflecting its inherent pairwise focus, ListNet consistently outperforms it in critical top-heavy metrics, including Recall@10 ($0.1935$) and NDCG@10 ($0.2434$). The fact that ListNet maintains superior ranking precision in a full-sort regime provides strong empirical support for our stability analysis: by explicitly modeling the listwise probability distribution, ListNet effectively suppresses the ranking noise that frequently impairs pairwise transfer in large-scale item spaces.
\section{Conclusion}

In this paper, we propose a unified theoretical framework to quantify the direct relationships between evaluation metrics, addressing the "Metric Mismatch" problem where offline gains fail to translate into online performance. Moving beyond surrogate-to-metric consistency, we provide an analytical framework to understand how performance trade-offs are governed by the underlying mathematical structures of metrics.

We categorize metrics into Pointwise, Pairwise, and Listwise groups and prove a strict structural hierarchy where ranking optimality inherently satisfies classification requirements. Through the regret transfer function, we identify a fundamental "Pointwise Transfer Failure," confirming that optimal classification provides zero stability for ranking quality. Furthermore, we establish that the coupling strength between Pairwise and Listwise metrics is scale-dependent, exhibiting a significant scaling asymmetry. These findings offer significant practical implications for machine learning system design, explaining why minor fluctuations in AUC can lead to disproportionate failures in top-heavy online objectives like NDCG. This provides a strong theoretical justification in large-scale recommender systems to ensure that offline improvements faithfully yield online value.

\bibliographystyle{unsrtnat}
\bibliography{sample-base}

\newpage
\appendix
\section{Detailed Definitions of Evaluation Metrics}
\label{app:metric_definitions}

\subsection{Pointwise Metrics}
Pointwise metrics focus on independent item identification, treating each instance as a standalone classification task. Let $n = |\mathscr{D}|$ denote the total number of instances, and $n^+ = \sum_{i=1}^n y_i$ and $n^- = n - n^+$ denote the number of positive and negative instances in the dataset, respectively. For truncation at position $k$, the weight incorporates an indicator function $\mathbb{I}(i \le k)$. $\phi(y) = y$ unless specifically defined.

\textbf{Accuracy (Acc)} measures the proportion of correctly classified instances, which relies on a decision threshold $\tau$ (typically $0.5$ for calibrated scores). In our unified framework, it can be viewed as a global metric where each position $i$ in the ranked list $\sigma_f$ is assigned a uniform weight,
\begin{equation}
    \text{Acc} = \frac{1}{n} \sum_{i=1}^n \mathbb{I}(y_{\sigma_f(i)} = \hat{y}_{\sigma_f(i)}), \quad  w(i, y, \mathscr{D}) = \frac{1}{n}, \quad \phi(y, \hat{y}) = \mathbb{I}(y = \hat{y})
\end{equation}
\textbf{Precision@$k$ (P@$k$)} calculates the fraction of relevant items within the top-$k$ positions of the list.
\begin{equation}
    \text{P}@k = \frac{1}{k} \sum_{i=1}^k y_{\sigma_f(i)} \quad  \quad w(i, y, \mathscr{D}) = \frac{1}{k} \mathbb{I}(i \le k)
\end{equation}
\textbf{Recall@$k$ (R@$k$)} evaluates the proportion of total positive instances captured within top-$k$ results.
\begin{equation}
    \text{R}@k = \frac{1}{n^+} \sum_{i=1}^k y_{\sigma_f(i)} \quad  \quad w(i, y, \mathscr{D}) = \frac{1}{n^+} \mathbb{I}(i \le k)
\end{equation}

\subsection{Pairwise Metric}
\textbf{AUC} measures the probability of correct pairwise ordering, which can be expressed as a linear sum over positive instances based on the number of negative instances ranked below them. Here we omit ties and focus on simplified tasks where
\begin{equation}
    \text{AUC} = \frac{\sum_{i:y_{\sigma_f(i)}=1} \sum_{j:y_{\sigma_f(j)}=0} \mathbb{I}(i < j)}{n^+ n^-},\quad  w(i, y, \mathscr{D}) = \frac{\sum_{j=i+1}^n (1 - y_{\sigma_f(j)})}{n^+ n^-}
\end{equation}
\subsection{Listwise Metrics}
These metrics are position-sensitive, assigning higher weights to items at the top of the ranked list to reflect user attention patterns. 

\paragraph{NDCG:} Discounted Cumulative Gain measures utility with logarithmic decay, where NDCG@$k$ is normalized by the Ideal DCG (IDCG).
\begin{align}
    \text{NDCG} &= \frac{1}{\text{IDCG}} \sum_{i=1}^n \frac{y_{\sigma_f(i)}}{\log_2(i+1)} & w(i,y,\mathscr{D}) &= \frac{1}{\text{IDCG} \cdot \log_2(i+1)} \\ 
    \text{NDCG}@k &= \frac{1}{\text{IDCG}_k} \sum_{i=1}^k \frac{y_{\sigma_f(i)}}{\log_2(i+1)} & w(i,y,\mathscr{D}) &= \frac{\mathbb{I}(i \le k)}{\text{IDCG}_k \cdot \log_2(i+1)} 
\end{align}

\textbf{MAP} averages the precision values at each rank where a relevant item is correctly identified. 
\begin{align}
    \text{MAP} &= \frac{1}{n^+} \sum_{i:y_{\sigma_f(i)}=1} \text{P}@i & w(i, y, \mathscr{D}) &= \frac{\sum_{j=1}^i y_{\sigma_f(j)}}{n^+ \cdot i} \\ 
    \text{MAP}@k &= \frac{1}{\min(k, n^+)} \sum_{i:y_{\sigma_f(i)}=1, i \le k} \text{P}@i & w(i, y, \mathscr{D}) &= \frac{\sum_{j=1}^i y_{\sigma_f(j)}}{\min(k,n^+) \cdot i} \cdot \mathbb{I}(i \le k) 
\end{align}

\textbf{MRR} focuses exclusively on the rank of the first relevant item in the list. 
\begin{align}
    \text{MRR} &= \sum_{i=1}^n \frac{y_{\sigma_f(i)}}{i} \prod_{j=1}^{i-1} (1 - y_{\sigma_f(j)}) & w(i, y, \mathscr{D}) &= \frac{\prod_{j=1}^{i-1} (1 - y_{\sigma_f(j)})}{i} \\ 
    \text{MRR}@k &= \sum_{i=1}^k \frac{y_{\sigma_f(i)}}{i} \prod_{j=1}^{i-1} (1 - y_{\sigma_f(j)}) & w(i, y, \mathscr{D}) &= \frac{\prod_{j=1}^{i-1} (1 - y_{\sigma_f(j)})}{i} \cdot \mathbb{I}(i \le k) 
\end{align}
\section{Proofs of Metrics Properties}
\subsection{Proof of Theorem~\ref{thm:intra_alignment} \& \ref{thm:truncation_monotonicity}}
\label{appdix:proof_intra}
To maintain consistency with our structural taxonomy, we prove the intra-group properties by following the sequence from Pointwise to Listwise groups.

\begin{lemma}[Bayes-Optimality of $\mathcal{G}_P$]
\label{lem:pointwise_logic}
Under the binary hypothesis $y \in \{0,1\}$:
\begin{enumerate}
    \item \textbf{Equivalence:} Conditioned on a fixed number of positives $n^+$ per list, P@$k$ and R@$k$ are Bayes-equivalent, i.e., $\mathrm{P}@k \equiv_{\mathscr{B}} \mathrm{R}@k$.
    \item \textbf{Inclusion:} The Bayes-optimal set of Acc is a subset of the truncated pointwise optimal sets, s.t. $\mathrm{Acc} \preceq_{\mathscr{B}} \mathrm{P}@k/\mathrm{R}@k$.
\end{enumerate}
\end{lemma}

\begin{proof}

1. Consider that $n^+$ is fixed for any given sequence of instances $\mathscr{D}$. The expected utilities are 
\begin{equation}
    U_{\mathrm{P}@k} = \mathbb{E}[\frac{1}{k}\sum_{i=1}^k y_i],\quad U_{\mathrm{R}@k} = \mathbb{E}[\frac{1}{n^+}\sum_{i=1}^k y_i].
\end{equation}
Since $1/k$ and $1/n^+$ are positive constants, both metrics are maximized by any $f$ that selects the set of $k$ instances with the largest $\eta(\bm{x})$, confirming $\mathscr{F}^*_{Prec@k} = \mathscr{F}^*_{Rec@k}$.

2. Acc requires correct classification for all $x \in \mathscr{D}$ relative to an optimal threshold. Achieving $\mathscr{F}^*_{Acc}$ implies that all positive instances are correctly identified and placed above all negative instances. This inherently satisfies the requirement for P@$k$/R@$k$ to fill the top-$k$ slots with positive instances, hence $\mathscr{F}^*_{Acc} \subseteq \mathscr{F}^*_{\mathrm {P}@k}=\mathscr{F}^*_{\mathrm{R}@k}$.
\end{proof}

As for listwise metrics, we prove their property with the following metric family.
\begin{definition}[Top-Weighted Additive Metrics]
\label{def:g_add}
A ranking metric $\mathcal{M}$ belongs to the class of \textbf{top-weighted additive metrics} ($\mathcal{G}_{add}$) if its expected utility $U_{\mathcal{M}}(f)$ can be expressed as:
\begin{equation}
    U_{\mathcal{M}}(f) = \mathbb{E}_{\mathscr{D}} \left[ \sum_{i=1}^n w(i) \cdot y_{\sigma_f(i)} \right]
\end{equation}
subject to the following conditions:
\begin{enumerate}
    \item \textbf{Position-Dependency:} The weights $\bm{w} = \{w(i)\}_{i=1}^n$ form a fixed, strictly decreasing sequence $w(1) > w(2) > \dots > w(n) \ge 0$.
    \item \textbf{Label-Independency:} Each weight $w(i)$ depends only on the rank position $i$ and is independent of the realization of labels in $\mathscr{D}$.
    \item \textbf{Linearity of Expectation:} The expected utility satisfies $U_{\mathcal{M}}(f) = \mathbb{E}_{\mathscr{D}} \left[ \sum_{i=1}^n w(i) \eta(\bm{x}_{\sigma_f(i)}) \right]$, where $\eta(\bm{x}) = \mathbb{E}[y | \bm{x}]$ is the conditional probability.
\end{enumerate}
\end{definition}

The structural properties of $\mathcal{G}_{add}$ ensure that the maximization of expected utility can be decoupled across rank positions. By leveraging the linearity of expectation and the monotonic decay of weights, we establish the following properties.

\begin{lemma}[General Properties of $\mathcal{G}_{add}$]
\label{lem:g_add_properties}
For any metric $\mathcal{M} \in \mathcal{G}_{add}$ with weights $w(1) > \dots > w(n) \ge 0$:
\begin{enumerate}
    \item \textbf{Optimal Ordering:} $U_{\mathcal{M}}(f)$ is maximized iff $f$ sorts instances such that $\eta(\bm{x}_{\sigma_f(1)}) \ge \dots \ge \eta(\bm{x}_{\sigma_f(n)})$.
    \item \textbf{Truncation Inclusion:} For its truncated version $\mathcal{M}_{@k}$, $\mathscr{F}^*_{\mathcal{M}} \subseteq \mathscr{F}^*_{\mathcal{M}_{@k}}$.
\end{enumerate}
\end{lemma}

\begin{proof}
(1) By Definition \ref{def:g_add}, $U_{\mathcal{M}}(f) = \mathbb{E}_{\mathscr{D}} [\sum w(i) \eta(\bm{x}_{\sigma_f(i)})]$. Since $\{w(i)\}$ is strictly decreasing, the rearrangement inequality states the sum is maximized iff $\eta(\bm{x}_{\sigma_f(i)})$ is non-increasing.

(2) Let $f \in \mathscr{F}^*_{\mathcal{M}}$. Since $f$ sorts the full list by $\eta$, it necessarily selects the $k$ largest $\eta$ values and sorts them non-increasingly in the top-$k$ positions, which is the optimality condition for $w_{@k}(i) = w(i)\mathbb{I}(i \le k)$.
\end{proof}

To bridge the abstract family $\mathcal{G}_{add}$ with specific ranking metrics, we provide the following proposition.

\begin{proposition}[Metric Embedding into $\mathcal{G}_{add}$]
\label{prop:embedding}
The convergence of $\mathcal{G}_L$ is guaranteed by the following embeddings:
\begin{enumerate}
    \item Global DCG is a direct instance of $\mathcal{G}_{add}$ with $w(i) = (\log_2(i+1))^{-1}$.
    \item NDCG, MAP and MRR, while not in $\mathcal{G}_{add}$, satisfy the Probability Ranking Principle (PRP)\citep{fuhr1992probabilistic}. Under binary relevance and item exchangeability, their Bayes-optimal sets $\mathscr{F}^*_{\mathrm{NDCG}}, \mathscr{F}^*_{\mathrm{MAP}}$ and $\mathscr{F}^*_{\mathrm{MRR}}$ converge to the same ordering rule as $\mathcal{G}_{add}$.
\end{enumerate}
\end{proposition}

\begin{proof}[Proof of Theorem \ref{thm:intra_alignment}]
The alignment (Equivalence) and Inclusion follow from Lemma \ref{lem:pointwise_logic} for Pointwise metrics and the combined results of Lemma \ref{lem:g_add_properties} and Proposition \ref{prop:embedding} for Listwise metrics.
\end{proof}

\begin{proof}[Proof of Theorem \ref{thm:truncation_monotonicity}]
\ 

\textbf{1. Optimal Set Nesting:} For $k_1 < k_2$, any $f \in \mathscr{F}^*_{\mathcal{M}_{@k_2}}$ satisfies optimality for the top-$k_2$ items. Since the top-$k_1$ items are a subset of the top-$k_2$, the condition for $\mathcal{M}_{@k_1}$ is inherently met, thus $\mathscr{F}^*_{\mathcal{M}_{@k_2}} \subseteq \mathscr{F}^*_{\mathcal{M}_{@k_1}}$.

\textbf{2. Regret Transfer Sensitivity:} Let $\Delta \eta_i \ge 0$ be the optimality gap. We distinguish between two cases based on the metric's normalization structure:

\begin{itemize}
    \item \textbf{Case 1: Pointwise Metrics (e.g., P@$k$, R@$k$):} 
    For $k_1 < k_2$, since $\text{Regret}_{@k_2}(f) = \frac{1}{k_2} \sum_{i=1}^{k_2} \Delta \eta_i \le \epsilon$, and $\Delta \eta_i \ge 0$, the partial sum is bounded by the total sum. Thus:
    \begin{equation}
        \text{Regret}_{@k_1}(f) = \frac{1}{k_1} \sum_{i=1}^{k_1} \Delta \eta_i \le \frac{1}{k_1} \sum_{i=1}^{k_2} \Delta \eta_i = \frac{k_2}{k_1} \left( \frac{1}{k_2} \sum_{i=1}^{k_2} \Delta \eta_i \right) \le \frac{k_2}{k_1} \cdot \epsilon
    \end{equation}
    Therefore, $\Psi_{k_2 \to k_1}(\epsilon) \le C \epsilon$ with $C = k_2/k_1$.
    
    \item \textbf{Case 2: Listwise Metrics (e.g., NDCG, MAP, MRR):} 
    For NDCG@$k$, letting $C = \mathrm{IDCG}_{k_2}/\mathrm{IDCG}_{k_1} \ge 1$ and noting that $w(i) \Delta \eta_i \ge 0$:
    \begin{equation}
        \text{Regret}_{@k_1}(f) = \frac{\sum_{i=1}^{k_1} w(i) \Delta \eta_i}{\mathrm{IDCG}_{k_1}} \le \frac{\sum_{i=1}^{k_2} w(i) \Delta \eta_i}{\mathrm{IDCG}_{k_1}} = \frac{\mathrm{IDCG}_{k_2}}{\mathrm{IDCG}_{k_1}} \cdot \text{Regret}_{@k_2}(f) \le C \epsilon
    \end{equation}
    For MAP and MRR, since they share the same $\mathscr{F}^*$ and their regret is monotonically controlled by the cumulative optimality gaps up to $k$, they exhibit a similar linear upper bound $\Psi_{k_2 \to k_1}(\epsilon) \le C \epsilon$.
\end{itemize}

Conversely, the reverse transfer $\Psi_{k_1 \to k_2}(\epsilon)$ is divergent. If $\text{Regret}_{@k_1}(f) = 0$, the gap $\Delta \eta_i$ for $i \in (k_1, k_2]$ can be arbitrarily large, resulting in $\Psi_{k_1 \to k_2}(0) = \sup \{ \text{Regret}_{@k_2}(f) \} > 0$, confirming information loss during upward truncation.
\end{proof}

\subsection{Proof of Theorem \ref{thm:inter_group_relations}}
\label{appdix:inter_bayes}
\subsubsection{Pointwise vs. Listwise}
\begin{proposition}
    Under the binary relevance hypothesis, the Bayes-optimal set for Listwise metrics is a proper subset of the Pointwise (Accuracy) optimal set, i.e., $\mathcal{M}_{L}\preceq_{\mathscr{B}}\mathcal{M}_P$.
\end{proposition}

\begin{proof}
Consider the Bayes-optimal risk minimizers in the space of all measurable functions $\mathscr{F}$. The set $\mathscr{F}^*_P$ consists of functions $f$ that satisfy the sign-consistency condition:
    \begin{equation}
        \mathscr{F}^*_P = \{ f \in \mathcal{F} \mid \forall \bm{x} \in \mathcal{X}, (f(\bm{x}) - \tau)(\eta(\bm{x}) - 0.5) > 0 \}
    \end{equation}
where $\tau$ is the decision threshold. This condition only constrains the mapping of $\bm{x}$ relative to the level set $\{\bm{x} : \eta(\bm{x}) = 0.5\}$ and imposes \textbf{zero constraints} on the relative ordering of any $\bm{x}_i, \bm{x}_j$ on the same side of the threshold.
    
By Lemma \ref{lem:g_add_properties}, any $\mathcal{M} \in \mathcal{G}_L$ requires $f$ to be a monotonic transformation of $\eta$:
    \begin{equation}
        \mathscr{F}^*_L = \{ f \in \mathcal{F} \mid \forall \bm{x}_i, \bm{x}_j \in \mathcal{X}, \eta(\bm{x}_i) > \eta(\bm{x}_j) \implies f(\bm{x}_i) > f(\bm{x}_j) \}
    \end{equation}
For any $f \in \mathscr{F}^*_L$, let $f(\bm{x}_{0.5}) = \tau$. The monotonicity ensures that $\forall \eta(\bm{x}) > 0.5, f(\bm{x}) > \tau$ and $\forall \eta(\bm{x}) < 0.5, f(\bm{x}) < \tau$. Thus, $\mathscr{F}^*_L \subseteq \mathscr{F}^*_P$. Note that the expected metric risk is governed by the conditional expectation $\eta(x)=P(y=1|x)$ rather than realizations $\mathscr{D}$. 

To show the inclusion is proper, we define the set of \textit{rank-reversed} predictors
\begin{equation}
\begin{aligned}
\mathscr{F}_{rev} = \{ f \in \mathscr{F}^*_P \mid \exists \bm{x}_i, \bm{x}_j \in \mathcal{X}, (\eta(\bm{x}_i) - 0.5)(\eta(\bm{x}_j) - 0.5) > 0, \\
\text{ s.t. } (\eta(\bm{x}_i) - \eta(\bm{x}_j))(f(\bm{x}_i) - f(\bm{x}_j)) < 0 \}.
\end{aligned}
\end{equation}
Since $\mathscr{F}_{rev}$ is non-empty for any continuous $\eta$, and $\mathscr{F}_{rev} \cap \mathscr{F}^*_L = \emptyset$ while $\mathscr{F}_{rev} \subset \mathscr{F}^*_P$, it follows that $\mathscr{F}^*_L \subsetneq \mathscr{F}^*_P$.
The existence of $\mathscr{F}_{rev}$ implies $\Psi_{P \to L}(0) > 0$, confirming that optimal classification provides no stability for ranking.
\end{proof}

\subsubsection{Pairwise vs. Listwise}
\begin{proposition}[Bayes-Optimal Convergence and Structural Heterogeneity]
\label{prop:structural_divergence}
The Bayes-optimal sets $\mathscr{F}^*_R$ (Pairwise) and $\mathscr{F}^*_L$ (Listwise) are equivalent, both converging to the set of monotonic transformations of the regression function $\eta$.
\end{proposition}

\begin{proof}
Consider the conditional expectation $\eta(\bm{x}) = \mathbb{E}[y|\bm{x}]$. The expected utility of a pairwise metric $\mathcal{M}_{R} \in \mathcal{G}_R$ is maximized when the predictor $f$ correctly orders all pairs $(\bm{x}_i, \bm{x}_j)$ such that $\eta(\bm{x}_i) > \eta(\bm{x}_j)$. Formally, $U_R(f)$ is maximized iff for every pair, $\text{sgn}(f(\bm{x}_i) - f(\bm{x}_j)) = \text{sgn}(\eta(\bm{x}_i) - \eta(\bm{x}_j))$.

By Lemma \ref{lem:g_add_properties}, any top-weighted additive metric $\mathcal{M}_L \in \mathcal{G}_L$ is maximized iff the instances are sorted in non-increasing order of their conditional expectations $\eta(\bm{x})$. Since both $U_R(f)$ and $U_L(f)$ reach their theoretical maximum iff $f$ preserves the partial ordering induced by $\eta$, we have:
    \begin{equation}
        \mathscr{F}^*_R = \mathscr{F}^*_L = \{ f \in \mathcal{F} \mid \forall \bm{x}_i, \bm{x}_j, \eta(\bm{x}_i) > \eta(\bm{x}_j) \implies f(\bm{x}_i) > f(\bm{x}_j) \}
    \end{equation}
    Thus, the two groups share the identical Bayes-optimal frontier with infinite model capacity.
\end{proof}

\subsection{Proof of Theorem~\ref{thm:ranking_to_pointwise}}
\label{appdix:ranking_to_pointwise}
\begin{proof}
Let $\mathscr{E} = \{i \mid (\eta_i - 0.5)(s_i - 0.5) < 0\}$ be the set of instances misclassified by $f$. The classification regret is given by $\text{Regret}_{\mathrm{Acc}} = |\mathscr{E}|/n$. 
For any misclassified instance $i \in \mathscr{E}$, there are two cases:
\begin{enumerate}
    \item $i$ is a false positive ($\eta_i < 0.5, s_i > 0.5$): For every true positive instance $j$ ($\eta_j > 0.5, s_j > 0.5$) that is correctly classified, the pair $(i, j)$ is correctly ranked. However, if any correctly classified negative instance $j$ ($\eta_j < 0.5, s_j < 0.5$) exists, it does not directly yield a ranking regret for $(i, j)$.
    \item The critical constraint comes from pairs $(i, j)$ where $i \in \mathscr{E}$ and $j \in \mathscr{D} \setminus \mathscr{E}$ such that $\eta_j$ and $\eta_i$ are on opposite sides of $0.5$. 
\end{enumerate}
Specifically, consider the total pairwise gap sum $\sum_{i < j} \mathbb{I}(\sigma(i) > \sigma(j)) (\eta_i - \eta_j)$. A classification error on instance $i$ implies that $i$ has crossed the $0.5$ threshold in score space. 
For a false positive $i \in \mathscr{E}$ ($\eta_i < 0.5$), it must be ranked above all correctly classified negative instances $j \in \mathscr{D}^- \setminus \mathscr{E}$ to not incur pairwise regret within the negative class, but more importantly, to satisfy $\text{Regret}_{\mathrm{Acc}} = 0$, it must not displace any positive instances. 

By summing the relevance gaps, each misclassified instance $i \in \mathscr{E}$ incurs a minimum regret of $|\eta_i - 0.5|$ relative to the threshold. In the Pairwise case, an error on $i \in \mathscr{E}$ implies it is misranked relative to at least one instance $j$ on the other side of the threshold. Summing over $\mathscr{E}$:
\begin{equation}
    \sum_{i \in \mathscr{E}} \delta \le \sum_{i \in \mathscr{E}} |\eta_i - 0.5| \le \frac{n^+ n^-}{1} \text{Regret}_{\mathrm{AUC}}(f; \mathscr{D})
\end{equation}
Dividing by $n \delta$, we obtain:
\begin{equation}\label{eq:auc_to_acc}
    \text{Regret}_{\mathrm{Acc}} = \frac{|\mathscr{E}|}{n} \le \frac{n^+ n^-}{n \delta} \text{Regret}_{\mathrm{AUC}}(f; \mathscr{D})
\end{equation}
For the Listwise case, a classification error at position $k$ incurs at least $w(k) \delta$ in the unnormalized sum. Using the lower bound $w(n)$ for any position:
\begin{equation}
    |\mathscr{E}| \cdot w(n) \cdot \delta \le \mathrm{IDCG}_n \cdot \text{Regret}_{\mathrm{NDCG}}(f; \mathscr{D}) \implies \text{Regret}_{\mathrm{Acc}} \le \frac{\mathrm{IDCG}_n}{n \delta w(n)} \text{Regret}_{\mathrm{NDCG}}
\end{equation}
As for the attainability of the upper bound, we construct a least favorable permutation where
\begin{itemize}
\item \textbf{AUC $\to$ Acc}: Consider a case where all instances satisfy $|\eta_i-0.5|=\delta$. If a single positive instance is displaced by a single negative instance, then $\mathrm{Regret}_{\mathrm{AUC}}=\frac{1}{n^+n^-}$. If this displacement causes one instance to cross the decision threshold, then $\mathrm{Regret}_{\mathrm{Acc}}=\frac{1}{n}$. The ratio directly matches the upper bound in Eq.(\ref{eq:auc_to_acc}).
\item \textbf{NDCG → Acc}: The ratio is maximized when a classification error incurs the minimum possible NDCG regret. This occurs when a positive item is misclassified and placed at the last rank $n$, where the weighting function $w(n)$ is at its minimum. 
\end{itemize}
Hence we finish the proof.
\end{proof}

\subsection{Proof of Theorem~\ref{thm:asymmetry}}
\label{appdix:proof_asymmetry}
\begin{proof}
We begin by establishing the relationship between the cumulative rank displacement and the respective regrets. Let $r_1 < r_2 < \dots < r_{n^+}$ be the ranks of the positive instances in the predicted permutation. The total pairwise rank inversion is defined as $I = \sum_{i=1}^{n^+} (r_i - i)$. 

\paragraph{I. Derivation of $\Psi_{R \to L}$:}
For the weighting function $w(r) = 1/\log(1+r)$, its first-order discrete difference $\Delta(t) = w(t) - w(t+1)$ is monotonically decreasing due to the convexity of $1/\log(1+r)$ for $r \ge 1$. Thus, for any rank $r_i \ge i$, the total weight loss for the $i$-th positive instance can be bounded by the maximum differential $\Delta_{\max}$:
\begin{equation}
    w(i) - w(r_i) = \sum_{t=i}^{r_i-1} (w(t) - w(t+1)) \le (r_i - i) \cdot \max_{t \ge 1} \Delta(t) = \Delta_{\max}(r_i - i).
\end{equation}
Summing across all $n^+$ positive instances, we obtain:
\begin{equation}
    \sum_{i=1}^{n^+} (w(i) - w(r_i)) \le \Delta_{\max} \sum_{i=1}^{n^+} (r_i - i) = \Delta_{\max} I.
\end{equation}
By definition, $\text{Regret}_{\mathrm{AUC}} = \frac{I}{n^+ n^-}$ and $\text{Regret}_{\mathrm{NDCG}} = \frac{\sum_{i=1}^{n^+} (w(i) - w(r_i))}{\sum_{i=1}^{n^+} w(i)}$. Substituting these into the inequality:
\begin{equation}
    \text{Regret}_{\mathrm{NDCG}} \cdot \sum_{i=1}^{n^+} w(i) \le \Delta_{\max} \cdot (n^+ n^- \cdot \text{Regret}_{\mathrm{AUC}}).
\end{equation}
As for attainability, this bound is tight when rank inversions occur exclusively at the very top of the list ($r_1=2,i=1$), where the marginal weight drop is exactly $\Delta_{\max}$. Solving for $\text{Regret}_{\mathrm{NDCG}}$ yields the transfer function $\Psi_{R \to L}(\epsilon) = \left( \frac{\Delta_{\max} \cdot n^+ n^-}{\sum_{i=1}^{n^+} w(i)} \right) \cdot \epsilon$.

\paragraph{II. Derivation of $\Psi_{L \to R}$:}
Similarly, we utilize the minimum differential $\Delta_{\min}$ to lower bound the weight loss. For any $r_i \ge i$, we have:
\begin{equation}
    w(i) - w(r_i) = \sum_{t=i}^{r_i-1} (w(t) - w(t+1)) \ge (r_i - i) \cdot \min_{1 \le t < n} \Delta(t) = \Delta_{\min}(r_i - i).
\end{equation}
Summing over all positive instances $i = 1, \dots, n^+$:
\begin{equation}
    \sum_{i=1}^{n^+} (w(i) - w(r_i)) \ge \Delta_{\min} \sum_{i=1}^{n^+} (r_i - i) = \Delta_{\min} I.
\end{equation}
Rearranging to isolate the rank inversion sum $I$:
\begin{equation}
    I \le \frac{1}{\Delta_{\min}} \sum_{i=1}^{n^+} (w(i) - w(r_i)).
\end{equation}
Substituting this into the expression for Pairwise regret:
\begin{equation}
    \text{Regret}_{\mathrm{AUC}} = \frac{I}{n^+ n^-} \le \frac{1}{n^+ n^- \cdot \Delta_{\min}} \sum_{i=1}^{n^+} (w(i) - w(r_i)).
\end{equation}
To express this in terms of Listwise regret, we multiply and divide by $\sum_{i=1}^{n^+} w(i)$:
\begin{equation}
    \text{Regret}_{\mathrm{AUC}} \le \left( \frac{\sum_{i=1}^{n^+} w(i)}{n^+ n^- \cdot \Delta_{\min}} \right) \cdot \frac{\sum_{i=1}^{n^+} (w(i) - w(r_i))}{\sum_{i=1}^{n^+} w(i)} = \left( \frac{\sum_{i=1}^{n^+} w(i)}{n^+ n^- \cdot \Delta_{\min}} \right) \text{Regret}_{\mathrm{NDCG}}.
\end{equation}
The bound is attained when rank inversions are concentrated at the bottom of the list ($r_{n^+} = n$), where the marginal weight loss is minimized to $\Delta_{\min}$. This derivation concludes the explicit form of $\Psi_{L \to R}(\epsilon)$.
\end{proof}

\subsection{Proof of Theorem~\ref{thm:rates}}
\begin{proof}
\label{appdix:proof_rates}
The rates are derived by substituting the asymptotic orders into the explicit coefficients from Theorem~\ref{thm:asymmetry}. For the Pairwise-to-Listwise coefficient $C_{R \to L} = \frac{\Delta_{\max} \cdot n^+ n^-}{\sum_{i=1}^{n^+} w(i)}$:
\begin{itemize}
    \item In the balanced case, $\sum_{i=1}^{n^+} w(i) \approx \int_1^n \frac{1}{\log(1+x)}dx \sim O(n/\log n)$. Thus, $C_{R \to L} \sim \frac{O(n^2)}{O(n/\log n)} = O(n \log n)$.
    \item In the imbalanced case, $\sum_{i=1}^{n^+} w(i)$ is a sum of $O(1)$ terms, hence $O(1)$. Thus, $C_{R \to L} \sim \frac{O(n)}{O(1)} = O(n)$.
\end{itemize}

For the Listwise-to-Pairwise coefficient $C_{L \to R} = \frac{\sum w(i)}{n^+ n^- \cdot \Delta_{\min}}$:
\begin{itemize}
    \item We approximate $\Delta_{\min} = w(n-1) - w(n) \approx |w'(n)| = \frac{1}{n \ln^2(n+1)} \sim O(\frac{1}{n \log^2 n})$.
    \item In the balanced case, $C_{L \to R} \sim \frac{O(n / \log n)}{O(n^2 \cdot \frac{1}{n \log^2 n})} = \frac{O(n / \log n)}{O(n / \log^2 n)} = O(\log n)$.
    \item In the imbalanced case, $C_{L \to R} \sim \frac{O(1)}{O(n \cdot \frac{1}{n \log^2 n})} = O(\log^2 n)$.
\end{itemize}
The derivation confirms that the Listwise-to-Pairwise path remains significantly more stable (logarithmic vs. polynomial growth) across both scenarios.
\end{proof}

\end{document}